\PassOptionsToPackage{preprint}{tmlr}
\documentclass[10pt]{article}

\usepackage{tmlr}

\usepackage[utf8]{inputenc}
\usepackage{amsmath}
\usepackage{mathtools}
\usepackage{amssymb}
\usepackage{amsthm}
\usepackage{graphicx}
\graphicspath{{figures/}}
\usepackage{booktabs}
\usepackage{array}
\usepackage{xcolor}
\usepackage{hyperref}

\def\month{MM}   % camera-ready only
\def\year{YYYY}  % camera-ready only

\hypersetup{
  colorlinks=true,
  linkcolor=blue!50!black,
  citecolor=blue!50!black,
  urlcolor=blue!50!black,
}

\theoremstyle{plain}
\newtheorem{theorem}{Theorem}
\newtheorem{lemma}{Lemma}
\newtheorem{proposition}{Proposition}
\theoremstyle{definition}
\newtheorem{assumption}{Assumption}

\title{Unlearning as Distribution Restoration:\\
A Controlled Counterfactual Study, a Validated Selective Screen,\\
and the Limits of Oracle-Free Certification}

\author{\name Sen Yang \email sy2576@stern.nyu.edu \\
      \addr Stern School of Business \\
      New York University
      \AND
      \name Yuen-Hei Yeung \email yy@nyu.edu \\
      \addr Courant Institute of Mathematical Sciences \\
      New York University}

\begin{document}
\maketitle

\begin{abstract}
Machine unlearning is commonly evaluated by matching a retrained oracle on trained
probes, but in a controlled nonce-fact testbed we find that this criterion can
favor methods that still \emph{retain held-out knowledge}: candidates it rates
adequate score held-out forget facts $-2.82$ nats (cluster-CI $[-3.16,-2.48]$)
\emph{below} the never-learned level. We recast unlearning as restoration to a
\textbf{matched retraining reference} and audit both oracle-free screens and
certificate-style criteria across $45$ model--seed cells spanning five open
architecture families. The retraining reference itself \textbf{falsifies an
absolute retain/round-trip certificate}: even the injected model, which retains
$R$ by construction, fails the fixed retain threshold in $41/45$ cells (i.e.\ passes in only $4/45$) and its
own round-trip in $31/45$, and the reference fully certifies in only $1/45$,
because the thresholds lie below the task's operating point and the round trip
carries a nonzero reacquisition floor that grows with the reacquisition budget. A
base-anchored held-out \textbf{screen} remains a strong \emph{selective necessary}
test: on a sealed challenge suite it rejects the injected model in $45/45$ cells
and accepts the reference in $44/45$, and it partially detects entity-routing
suppression through a distributional seam ($35/45$). A \textbf{damage-relative
recalibration}, with thresholds set by the reference's own operating point and the
replication noise of independent retraining draws, certifies a small candidate
subset in $15/45$ cells (concentrated in the noisiest-replication family);
where it does not abstain, its picks lie within retraining noise ($0.80$ nats) on
the held-out axes it optimizes, while the common trained-probe criterion sits
$5.17$ nats away on axes it never optimized (same cells; a supporting comparison,
not a head-to-head benchmark). Finally, adversarial challenge models show
that \emph{forward-only certification is not sound}: a fixed-magnitude logit
suppression attack ($-10$ on the forget-answer logits) defeats the full forward battery in $12/45$ cells (Clopper--Pearson upper bound $39.6\%$),
motivating a finite-query impossibility boundary and scoping our method as an
\emph{empirical selective test for methods-as-produced} rather than a universal
certificate. An identifiability theorem delimits which facts admit an oracle-free
forget threshold at all, with TOFU as the predicted boundary case.
\end{abstract}

% -------------------------------------------------------------------
\section{Introduction}
\label{sec:intro}

Unlearning is evaluated by forget-efficacy and retain-utility, but both can be
satisfied by \emph{suppression} that diverges from the true retrained model. The
retrained (``retain-only'') model is the gold reference but is usually
unavailable. We ask two questions: (i)~\emph{what does good unlearning look like
relative to a matched retraining reference?} and (ii)~\emph{can we tell good from
bad unlearning without that reference?} A distinctive feature of our answer is
that we turn the audit on our own instruments: we run the retraining reference
itself through every criterion we propose, and report which criteria it can and
cannot pass.

\medskip\noindent Our contributions are:
\begin{enumerate}
  \item \textbf{Protocol critique.} In a controlled dose-matched injection testbed
  with a matched counterfactual retraining reference, the common trained-probe
  oracle-KL criterion \emph{rewards residual knowledge}: candidates it rates
  adequate score held-out forget facts $-2.82$ nats (cluster-CI $[-3.16,-2.48]$)
  below the never-learned level, and its preferred selections sit $5.17$ nats from
  the reference's held-out position.
  \item \textbf{Absolute certification falsified by its own reference.} Fixed
  retain/round-trip thresholds sit below the injected model's operating point
  (it fails them in $41/45$ and $31/45$ cells); the retraining reference certifies
  in $1/45$. The round-trip floor is largely reacquisition-budget drift
  (three of five families close their own loop at short budgets), with one family
  showing intrinsic path-dependence.
  \item \textbf{A validated selective screen.} A base-anchored held-out rank
  screen, audited on a sealed challenge suite of known-label models: it rejects
  the injected model in $45/45$ cells, accepts the reference in $44/45$, and
  partially detects entity-routing suppression through a distributional seam
  ($35/45$). It is a \emph{necessary} test with measured sensitivity, not a
  sufficiency certificate.
  \item \textbf{Damage-relative recalibration.} Re-anchoring thresholds to the
  reference's own operating point, with tolerances set by the replication noise of
  independent retraining draws, yields a selective partial positive: certification
  in $15/45$ cells, abstention elsewhere, and non-abstaining picks within
  retraining noise of the reference on held-out axes ($0.80$ vs $5.17$ for the
  trained-probe criterion on the same cells; a supporting comparison, not a
  head-to-head benchmark, since the selector optimizes the recalibrated axes).
  \item \textbf{An adversarial boundary.} On the same challenge suite,
  a fixed-magnitude logit suppression attack defeats the full forward battery in $12/45$
  cells (Clopper--Pearson upper bound $39.6\%$): forward-only certification is not
  adversarially sound, so our guarantee is scoped to \emph{methods-as-produced},
  and a finite-query impossibility argument explains why.
  \item \textbf{An identifiability theorem} delimiting when oracle-free selection
  is possible at all, with TOFU as its predicted boundary, plus a finite-sample
  certifiability bound on any rank-test forget screen.
\end{enumerate}

% -------------------------------------------------------------------
\section{Controlled Testbed with a Matched Retraining Reference}
\label{sec:testbed}

\paragraph{Injection.} From a base model $M_0$, we continue-pretraining to
co-acquire a \textbf{forget set $F$} and a \textbf{retain-acquired set $R$} of
nonce facts (dose-matched; wikitext retain mix), yielding $M_{\text{inj}}$. We
reuse the controlled nonce-injection \emph{substrate} introduced for
acquisition-diversity experiments; the present paper studies a distinct
evaluation problem (whether unlearning candidates restore the matched
retraining distribution) with a different intervention variable, target
estimand, and failure mode. To avoid double-counting across that companion work,
we note that the shared element is the injection methodology only: the $45$
model--seed cells, the $820$ unlearning candidates, the forget/retain facts, and
the held-out and sealed challenge (\textsc{audit}; Section~\ref{sec:panel}) probe pools of this paper are constructed for
restoration evaluation and are not reused as, nor counted as, replications of any
acquisition-diversity result (though both lines use \texttt{gate\_proj} masks,
here as one \emph{unlearning} candidate family, there as a \emph{deletion-cost}
instrument).

\paragraph{Matched retraining reference.} $M_{\text{oracle}}$ replays the
identical stream \textbf{without $F$} (inject $R$ only): a model that never saw
$F$ but has the same retain competence. We call it a \emph{matched retraining
reference} rather than an ``exact oracle'': bf16 training is nondeterministic,
and rebuilding the reference from independent draws shifts per-fact held-out
deltas by up to $\delta_{\mathrm{equiv}} = 0.93$ nats at the P90 (per-family
range $0.82$--$2.92$; Section~\ref{sec:audit}). Any single draw is therefore one
realization of retraining, not a unique counterfactual, and
$\delta_{\mathrm{equiv}}$ is the scale below which two retrainings are
indistinguishable, a quantity the rest of the paper uses repeatedly.

\paragraph{Why mixed acquisition.} If the stream contained only $F$, full
checkpoint-rollback ($=M_0$) would trivially ``forget''; co-acquiring $R$ makes
rollback destroy retain (retain-NLL $6.6$ vs oracle $2.4$), so a method must do
real work.

\paragraph{Metrics.} \emph{trained-probe oracle-KL} $= \mathrm{KL}(\text{candidate}
\,\|\, M_{\text{oracle}})$ on $F{+}R$ \emph{trained} probes, the criterion the
field uses as ground truth, which Section~\ref{sec:oraclefail} shows can reward
residual knowledge; \emph{held-out} forget-NLL on $F$ paraphrases never used in any
training or unlearning loss (the surface our screens operate on); retain-NLL on
$R$; general damage on wikitext. We keep these roles separate throughout: the
trained-probe criterion is \emph{audited}, not assumed; held-out deltas anchored to
$M_0$ are the primary evaluation surface; and forget-NLL / retain-NLL alone are
diagnostics a suppression method can also satisfy.

\paragraph{Models.} Qwen2.5-0.5B/1.5B, SmolLM2-360M/1.7B, Pythia-410M/1.4B,
GPT-2-medium/large, OLMo-2-1B (five families; SwiGLU and GELU). Real unlearning
methods (below) are architecture-agnostic; the \texttt{gate\_proj} mask variant
applies only to SwiGLU models. We deliberately span diverse small-to-mid open
families rather than a single large model: the audit requires a \emph{matched
retraining reference} (and its independent redraws) and a re-acquire round trip
for \emph{every} model--seed cell ($45$ cells), which is feasible at this scale
and gives cross-architecture evidence; scaling the evaluation to larger
instruction-tuned models is future work. TOFU (on Phi-1.5) is included not as a comprehensive
real-world unlearning sweep but as the predicted \emph{boundary} case for
identifiability (Section~\ref{sec:identifiability}).

% -------------------------------------------------------------------
\section{Distribution Restoration as the Unlearning Target}
\label{sec:diagnostic}

\paragraph{Methods compared.} Each starts from $M_{\text{inj}}$, at matched
compute, with full forget/retain Pareto fronts: path-specific \textbf{KL
reversion} (forget${\to}M_0$, retain${\to}M_{\text{inj}}$), \textbf{task-vector
negation} ($M_{\text{inj}} - c\cdot(M_{F\text{only}}-M_0)$), \texttt{gate\_proj}
\textbf{mask} reversion; versus \textbf{gradient ascent (GA)} and
\textbf{NPO}~\citep{zhang2024npo}, each $+$retain; plus \textbf{checkpoint
rollback} ($=M_0$).

\paragraph{Result (controlled, matched forgetting).} Restoration methods
(KL-reversion, task-vector) reach oracle-KL ${\approx}\,2$--$4$; suppression
(GA/NPO) ${\approx}\,5$--$7$ even after a \textbf{$32$-configuration fair sweep}
(lr $\times$ $\beta$ $\times$ retain-weight): they over-forget (forget-NLL far
past the oracle) and collapse retain (retain-NLL ${\to}\,0 \ll$ oracle), i.e.\
they \emph{distort} rather than restore. Checkpoint rollback destroys $R$. Across
$3$ seeds, at matched forgetting restoration is ${\sim}2\times$ closer to the
oracle than suppression (Table~\ref{tab:diag}).

\paragraph{Interpretation.} Deletion that approximates the retrained model is
\emph{restoration} of the pre-acquisition distribution on $F$ while preserving the
post distribution on $R$, not indiscriminate loss-increase.

\subsection{The trained-probe criterion rewards residual knowledge}
\label{sec:oraclefail}

The diagnostic above uses proximity to the retraining reference as measured on
\emph{trained} probes: the phrasings that appeared in injection and unlearning
losses. Auditing that criterion against \emph{held-out} paraphrases exposes a
systematic failure: across the $45$-cell matrix, candidates the trained-probe
criterion rates adequate score held-out forget facts a mean of $\mathbf{-2.82}$
nats (cluster-bootstrap CI by cell $[-3.16, -2.48]$, entirely below zero;
per-fact median $-2.91$, $n=948$) \emph{below} the never-learned level. They
demonstrably retain paraphrase-recoverable knowledge of $F$ while matching the
reference on the phrasings the criterion happens to query. Matching the retrained
model on trained probes is therefore \emph{not} evidence of restoration. That
observation motivates every held-out, base-anchored instrument in the rest of
this paper.

\paragraph{The method dichotomy.} The same held-out surface splits the method
space cleanly. Loss-based methods (KL-reversion, GA, NPO) suppress only the
phrasings in their loss: $77$--$87\%$ of their candidates are excluded by the
held-out screen as \emph{retaining} $F$. Weight-space methods (\texttt{gate\_proj}
mask, task-vector) generalize forgetting across phrasings but damage retain.
Under the absolute (fixed-threshold) bar (screen $\wedge$ retain $\le 0.5$ $\wedge$
round-trip $\le 0.5$), only $3/820$ candidates matrix-wide satisfy all three,
an outcome that Section~\ref{sec:audit} traces to the bar itself rather than
to the methods.

\begin{figure}[t]
\centering
\includegraphics[width=0.62\linewidth]{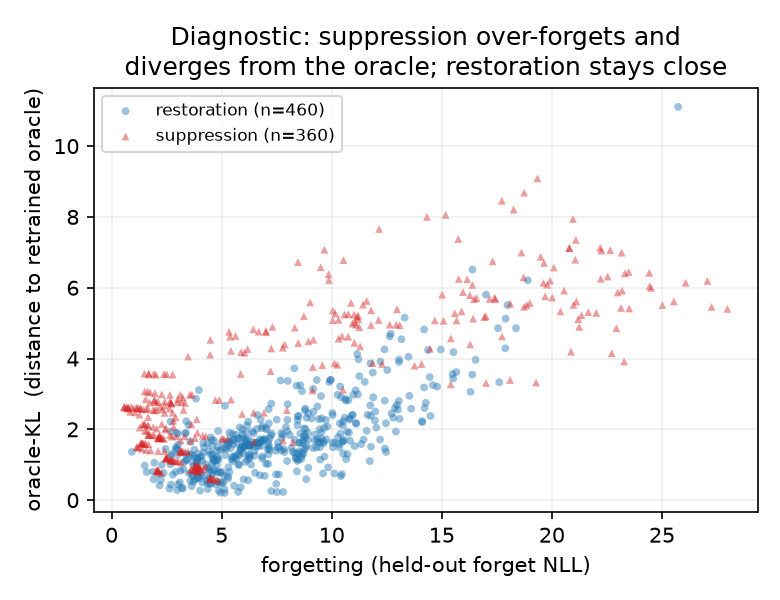}
\caption{Diagnostic (pooled candidates, 45 matrix cells). Suppression methods
(GA/NPO) over-forget \emph{and} diverge from the matched retraining reference (high
trained-probe KL); restoration methods (task-vector, KL-reversion) stay close to the
oracle across the forgetting range.}
\label{fig:diag}
\end{figure}

\begin{table}[t]
\centering
\caption{Diagnostic: oracle-KL of restoration vs suppression methods at matched
forgetting, in the controlled testbed (across $3$ seeds and a $32$-configuration
fair sweep) and on TOFU's real retrained reference. Lower is closer to the retraining reference. Restoration stays ${\sim}2\times$ closer to the oracle throughout.}
\label{tab:diag}
\begin{tabular}{lcc}
\toprule
setting & restoration (KL-reversion, task-vector) & suppression (GA, NPO) \\
\midrule
controlled (oracle-KL) & ${\approx}\,2$--$4$ & ${\approx}\,5$--$7$ \\
TOFU (oracle-KL)       & ${\approx}\,0.1$    & $0.7$--$1.2$ \\
\bottomrule
\end{tabular}
\end{table}

% -------------------------------------------------------------------
\section{Oracle-Free Round-Trip Certificate}
\label{sec:certificate}

\paragraph{Idea.} A genuine counterfactual deletion closes the loop
\emph{acquire $\to$ delete $\to$ re-acquire}: after re-acquiring $F$, a
truly-restored model returns to ${\approx}\,M_{\text{inj}}$ on retain; a
suppression model's collateral damage \textbf{persists}. We define the oracle-free
\textbf{round-trip residual}
\[
  \mathrm{cert}(C) = \mathrm{KL}\!\bigl(\mathrm{reacquire}_F(C) \,\big\|\,
  M_{\text{inj}}\bigr) \text{ on retain } + \text{ wikitext},
\]
which uses only $M_{\text{inj}}$ and the forget data, with no oracle.

\paragraph{Forget-qualification: the base-anchored screen (necessary).} The
round-trip residual ranks \emph{restoration quality}; it must be paired with a
check that the candidate actually forgot. The principled form is an
\textbf{under-forgetting exclusion test} on \emph{base-anchored deltas}, and we
state explicitly which model scores what, because two subtly different designs are
wrong. For candidate $C$, forget fact $i$, and never-learned probe fact $j$
(facts of the same nonce type from unused seeds, never injected anywhere), define
\[
  \Delta^F_i = \mathrm{NLL}_{C}(F_i) - \mathrm{NLL}_{M_0}(F_i),
  \qquad
  \Delta^P_j = \mathrm{NLL}_{C}(P_j) - \mathrm{NLL}_{M_0}(P_j),
\]
both on \emph{held-out} paraphrases and both scored by the \emph{candidate
itself}, anchored to the base model $M_0$. A one-sided exact rank (Mann--Whitney)
test of $\{\Delta^F_i\}$ against $\{\Delta^P_j\}$, Holm-corrected at
$\alpha=0.05$ over the candidate pool of the cell, excludes candidates whose
forget deltas are significantly \emph{below} their own probe deltas (residual
knowledge). Anchoring to $M_0$ matters twice over: $M_0$ never learned $F$
\emph{or} $P$, so it is symmetric in the two sets and oracle-free, and the
per-fact anchoring blocks two failure modes of naive designs.
Scoring probes with $M_{\text{inj}}$ compares NLLs \emph{across models}, so
global damage masquerades as forgetting, while anchoring deltas to
$M_{\text{inj}}$ injects the very $F$-vs-$P$ label under test (its baseline gap
between learned $F$ and never-seen $P$ \emph{is} the label). Probes and the
retraining reference are never used to score candidates.

\paragraph{When the check is possible at all.} This test is usable only if the
design supplies enough probes, and the constraint is combinatorial rather than
statistical:

\begin{proposition}[Finite-sample certifiability]
\label{prop:certifiability}
Fix a cell with $K$ candidates, $m$ forget items, and $n$ matched never-learned
probes. For the one-sided \emph{exact} rank test, the smallest attainable
$p$-value is $1/\binom{m+n}{m}$. Holm's step-down rejects its most significant
hypothesis only if that $p$-value is at most $\alpha/K$. Hence \emph{no candidate
can be excluded, at any effect size}, unless
\[
  \binom{m+n}{m}\;\ge\;K/\alpha .
\]
\end{proposition}

\noindent The bound is a property of the test's null lattice, not of the data: it
holds however badly a candidate under-forgets. With $\alpha=0.05$, our pools
require $\binom{m+n}{m}\ge400$ in SwiGLU cells ($K=20$) and $\ge320$ in GELU cells
($K=16$, which lack the \texttt{gate\_proj} mask arm). At $m=n=4$, $\binom{8}{4}=70$ falls short of both, so the exact untied screen is
\textbf{vacuous by construction} (it excludes $0$ of $820$ candidates). The matrix
therefore uses $n=16$ probes at the same dose ($m=4$): $\binom{20}{4}=4845$ clears both pools, the attainable
floor is $1/4845 \approx 2.06\times10^{-4} < \alpha/K$, and the powered screen
excludes $498$ of $820$ candidates. We emphasize the correct reading of the
bound: it guarantees attainable $p$-value \emph{resolution} sufficient for Holm
rejection, not power against a specified effect size. TOFU supplies $m=n=40$,
where $\binom{80}{40}\approx10^{23}$ clears the bound by twenty orders of
magnitude, and the screen is decisive there ($16/16$ candidates excluded in every
seed).

\paragraph{Probes are the free knob.} Read as a design rule, the bound separates
two costs that are easy to conflate. Enlarging
the forget set raises $m$, but it also changes the injection dose, so it perturbs
the very system under study. Enlarging the probe set raises $n$ without touching
that system: never-learned probes enter only as forward passes against
$M_{\mathrm{inj}}$, and never enter injection, unlearning, or the residual
$\hat c$. Raising $n$ is not statistically inert: a larger probe set changes
the sampled reference distribution the forget losses are ranked against, so it
changes the comparison and not merely the resolution of the lattice. Even so, it
leaves the intervention, the candidates, and every candidate-side quantity
untouched. If probes are matched one-per-forget-item ($n=m$), the requirement
$\binom{2m}{m}\ge K/\alpha$ forces $m\ge6$; holding our dose fixed at $m=4$,
$n=8$ is the smallest probe count sufficient in both pools
($\binom{11}{4}=330$ clears $320$ but not $400$; $\binom{12}{4}=495$ clears both);
we adopt $n=16$, comfortably above the minimum, and store per-item values so every
$n\le16$ sub-panel is recoverable by subsetting. The asymmetry is the practical
content of the proposition: a screen can be underpowered purely because probes
were treated as scarce when they were free.

\paragraph{Implementation audit.} Two details make the bound auditable rather than
nominal. First, the exact lattice is what our implementation actually evaluates in
the regime the bound is about: we verified exactness at $m=4$ for every $n\le16$
we use, and in the matched case up to $m=8$, beyond which it substitutes a normal
approximation. At $m=4$, $n=16$ the attainable minimum is exactly
$1/\binom{20}{4} = 2.06\times10^{-4}$, which is also the smallest $p$-value
observed anywhere in the $820$ candidates, attained by $238$ of them. This is a
data-level audit of the lattice: the test operates on \emph{fact-level} scalars
(held-out template NLLs are averaged \emph{within} each fact before ranking), and
a paired per-fact sign design, whose floor is $2^{-4}=0.0625$, could not produce
these $p$-values. Second, ties are the one caveat: they push the implementation
onto the approximation, which is \emph{not} floored by $1/\binom{m+n}{m}$, so in
tied cases the bound constrains the exact test rather than the code path taken;
across all $820$ candidates no tied case attained a smaller minimum. TOFU's
$m=n=40$ likewise falls in the approximation's regime, so its
$p\approx7\times10^{-15}$ is the approximation's rather than the exact lattice's.
It is far below $\alpha/K$ either way, so the screen's decisiveness there does
not rest on the exact argument.

\paragraph{Selective certification (abstention).} We therefore treat
certification as \emph{selective}. The selector returns a certified choice only
when the forget-qualification threshold is identifiable from the never-learned
reference (Section~\ref{sec:identifiability}), when
Proposition~\ref{prop:certifiability} is satisfied so the screen can reject at
all, and when a candidate clears the reference-relative bar of
Section~\ref{sec:recalib}; otherwise it reports \textbf{uncertified} rather than
silently degrading to residual-only ranking. The failure modes are distinct and
we observe each: TOFU fails \emph{identifiability} (the reference is
unreachable, so \emph{every} candidate is excluded; Section~\ref{sec:tofu}); the
underpowered $n=4$ design fails \emph{power} (resolved at $n=16$);
and the recalibrated bar abstains in $30/45$ matrix cells. A certificate that
cannot say ``uncertified'' would report these cases as success.

\paragraph{Screen-pass is not restoration (equivalence).} Because the screen is
one-sided, passing it means only ``not demonstrably retaining'': non-rejection is
not equivalence to the never-learned state, and with a weak test it would be
vacuously easy. We quantify the gap. Among the $322$ screen survivors of the
matrix, the median absolute gap between forget and probe deltas is
$\mathbf{1.93}$ nats (cell-clustered bootstrap CI $[1.68, 2.34]$), well above the
retraining-replication tolerance $\delta_{\mathrm{equiv}}=0.93$; only
$28.0\%$ of survivors (CI $[22.6\%, 34.0\%]$) are TOST-equivalent to
never-learned at that tolerance. The screen is a \emph{necessary} filter, and the
recalibrated bar of Section~\ref{sec:recalib} therefore adds an explicit
equivalence axis rather than treating survival as success.

\paragraph{Selector (oracle-free, at a glance).} Given candidate unlearned models
$\{C\}$ (the \emph{delete} outputs of the methods under comparison), each starting from
$M_{\text{inj}}$:
\begin{enumerate}\itemsep2pt
  \item \textbf{Re-acquire:} fine-tune each $C$ on the forget set $F$ to convergence,
    giving $\mathrm{reacquire}_F(C)$.
  \item \textbf{Measure residual:} $\mathrm{cert}(C) = \mathrm{KL}\!\bigl(\mathrm{reacquire}_F(C)\,\|\,M_{\text{inj}}\bigr)$
    on retain $+$ wikitext.
  \item \textbf{Qualify:} exclude any $C$ that fails the under-forgetting test against
    matched never-learned probes (residual knowledge).
  \item \textbf{Select:} among survivors, return $\arg\min_C \mathrm{cert}(C)$.
\end{enumerate}
No oracle, gold posterior, or probe labels enter the selection; only $M_{\text{inj}}$
and the forget data are used.

\paragraph{Selector cost.} The certificate's cost is a per-candidate
re-acquisition fine-tune (each $C$ is briefly fine-tuned on $F$ to convergence),
so selector compute scales linearly in the number of candidates times the
re-acquisition cost; it uses no retrained oracle and no held-out probes at
selection time. This makes it suited to offline method-selection studies rather
than per-query deployment.

\paragraph{Scope of the guarantee (what ``certificate'' does and does not mean).}
We use \emph{certificate} in a deliberately narrow, selection-theoretic sense, and
state the scope explicitly because the term is easily over-read. (i)~\emph{It is a
selector, not a deletion or privacy guarantee}: it ranks a \emph{fixed, given pool}
of candidate unlearned models and says which one is closest to the counterfactual
retrained model; it certifies nothing about an individual model in isolation, and
makes no differential-privacy or right-to-erasure claim. (ii)~\emph{Offline
setting}: the evaluator holds $M_{\text{inj}}$, the forget data $F$, and the
candidate pool, and may fine-tune; the retrained oracle is never available. (iii)
\emph{Non-adaptive methods}: candidates are produced without knowledge of the
certificate. A bilevel adversary that optimizes a candidate \emph{against} the
round-trip residual is outside the present analysis and is the most natural attack
on the construction (Section~\ref{sec:lim}). (iv)~\emph{Assumption-bound}:
oracle-consistency holds under the threshold-separation and residual-soundness
assumptions stated in Section~\ref{sec:identifiability}, and only where the
forget-qualification threshold is identifiable from the observable never-learned
reference; the finite-sample test is an under-forgetting screen, not a proof of
adequacy for any single candidate.

% -------------------------------------------------------------------
\section{The Matched-Retraining Audit: the Absolute Bar Fails Its Own Reference}
\label{sec:audit}

The natural certificate pairs the screen with absolute thresholds (retain
damage $\le 0.5$ nats and round-trip residual $\le 0.5$), and under that frozen
bar only $3/820$ candidates certify matrix-wide
(Section~\ref{sec:oraclefail}). The bar must first survive its own reference: we run \textbf{the
retraining reference, and the injected model, through the full battery,
component-wise, in every cell.}

\paragraph{Component-wise audit of the reference and injected model.} The reference passes the base-anchored
screen in $43/45$ cells, the one component it reliably clears. It passes
\emph{retain} $\le 0.5$ in only $11/45$ and the round-trip threshold in $7/45$;
it fully certifies in $\mathbf{1/45}$ cells. The diagnosis is not that retraining
is poor but that \textbf{the thresholds sit below the task's operating point}:
$M_{\text{inj}}$ itself, which retains $R$ perfectly by construction,
passes the absolute retain threshold in only $\mathbf{4/45}$ cells (median
retain-NLL $1.20$, range $0.24$--$4.45$: models simply do not memorize the nonce
retain facts below ${\sim}1$ nat at this dose), and $M_{\text{inj}}$'s
\emph{own} round-trip self-closure (re-acquiring $F$ into the model that
already knows $F{+}R$ and measuring the KL back to itself) exceeds $0.5$ in
$\mathbf{31/45}$ cells (median floor $0.92$). A candidate was being asked to pass
thresholds that the model it restores \emph{toward} cannot pass. Under a
consequence table frozen before the run, this outcome retires the absolute bar:
``no method certifies'' was a statement about the bar, not the methods.

\paragraph{Replication noise of the reference.} Rebuilding the reference
from three independent draws in one representative cell per family gives the
per-family indistinguishability tolerance $\delta_f$ (P90 per-fact held-out
delta gap): $0.82$--$2.92$ nats across families, median $0.93$. Draw-to-draw
spreads of retain-NLL and round-trip residual supply the reference-noise margins
$\varepsilon_f$ used below. The wide family spread is itself a finding: a single
global tolerance mis-scales at least some families, so all reference-relative
results below are computed per family, with a global-minimum sensitivity row.

\paragraph{The round-trip floor is largely budget drift.} Re-measuring
$M_{\text{inj}}$'s self-closure at reacquisition budgets
$\{15, 30, 45, 75\}$ steps shows the floor rising near-monotonically with budget
in every family; at the shortest budget three of five families close their own
loop below $0.5$ (gpt2 $0.05$, SmolLM2 $0.12$, Pythia $0.44$), while OLMo-2
never does (floor $1.4$--$2.1$ at every budget: intrinsic path-dependence).
The round-trip residual at a frozen budget is therefore dominated by
\emph{reacquisition drift} common to every model put through the protocol.
This is why the recalibration below subtracts the per-cell floor:
candidates and floor share the budget, so the subtraction removes the
common-mode drift.

% -------------------------------------------------------------------
\section{A Validated Selective Screen and Its Adversarial Boundary}
\label{sec:panel}

A screen worth trusting must be validated as an instrument: measured true-accept
and false-accept rates on models whose ground truth is known by construction. We
freeze a sealed \textsc{audit} template pool (never used in any loss,
calibration, or selection; used exactly once, here) and pass a known-label
challenge panel through the battery blind, in every cell:
the base model $M_0$; the injected model $M_{\text{inj}}$; the retraining
reference; \textbf{logit suppression} (an additive $-10$ penalty on the
forget-answer token logits of $M_{\text{inj}}$: pure suppression with the
knowledge intact); an \textbf{entity router} (dispatch inputs containing a
forget-subject token sequence to $M_0$, all else to $M_{\text{inj}}$;
suppression that generalizes across phrasings); \textbf{embedding corruption}
(zeroing the forget subject/object embedding rows); and weight-space
interpolations $\lambda\, M_{\text{oracle}} + (1{-}\lambda) M_{\text{inj}}$.

\paragraph{Sensitivity and true-accept.} The screen rejects $M_{\text{inj}}$ in
$\mathbf{45/45}$ cells and accepts the reference in $\mathbf{44/45}$; acceptance
along the interpolation path rises monotonically toward the reference. On the
corrected screen the residual-based selection also never picks a suppression
method in the matrix ($0/45$; exact one-sided binomial upper bound $6.4\%$
per-cell, $28.3\%$ clustered by the nine architectures; we report the bound,
not ``never'').

\paragraph{A distributional seam partially detects routing.} The entity router
is behaviorally reference-like on every queried input, yet the screen rejects it
in $\mathbf{35/45}$ cells. The mechanism is a seam the router cannot avoid: its
forget-side responses come from $M_0$ (deltas exactly zero) while its probe-side
responses come from $M_{\text{inj}}$, whose drift on never-learned probes is
nonzero, and the base-anchored rank test reads that inconsistency as residual
structure. We call this \emph{base-delta drift-seam sensitivity}: the screen is
strictly stronger against routing than a naive paraphrase test, though the
router still passes in $10/45$ cells, a concrete, quantified witness for the
finite-query impossibility argument of Section~\ref{sec:identifiability}.

\paragraph{The adversarial boundary: forward-only certification is not sound.}
Logit suppression exposes the limit. The screen \emph{alone} false-accepts it in
$28/45$ cells (by design: a one-sided under-forgetting test cannot
penalize looking \emph{more} forgotten), and the equivalence axis of
Section~\ref{sec:recalib} catches most of those; but in $\mathbf{12/45}$ cells
the fixed-magnitude penalty happens to land the forget-answer NLL within the
family tolerance of the reference's position, and the \emph{entire forward
battery} accepts a model whose knowledge is intact. We report $12/45$
descriptively and by family (Qwen and OLMo cells dominate the evasions): the
$45$ cells are fixed, family-clustered, and heterogeneous, so this is not an
i.i.d.\ success rate; the one-sided
Clopper--Pearson upper bound ($39.6\%$) characterizes \emph{these} cells, not a
guarantee about future models. Catching such an attack requires the relearning
axes (the suppressed model relearns instantly) or white-box access, neither
available to a forward-only protocol. Embedding corruption, by contrast, is rejected almost everywhere on
retain damage (compatible with $M_{\text{inj}}$ in only $3/45$ cells): in a
nonce-token testbed, deleting embedding rows is a blunt distributional
intervention, not surgical suppression. The consequence is stated plainly in
every claim of this paper: our certification is an \textbf{empirical selective
test for methods-as-produced}, not an adversarially sound certificate.

% -------------------------------------------------------------------
\section{Damage-Relative Recalibration: a Selective Partial Positive}
\label{sec:recalib}

The audit's constructive consequence is a recalibrated bar whose rules were
frozen before the analysis ran: measure candidates \emph{relative to the
reference's operating point}, at the tolerance retraining itself cannot beat.
A candidate $C$ in a cell of family $f$ is \textsc{certified} iff
(i)~it passes the base-anchored screen; (ii)~\emph{forget equivalence}:
$|\mathrm{med}(\Delta^F(C)) - \mathrm{med}(\Delta^F(\text{reference}))| \le
\delta_f$, against the \emph{same-draw} reference arrays of that cell;
(iii)~\emph{retain damage} $\mathrm{NLL}_R(C) - (\mathrm{NLL}_R(M_{\text{inj}})
- \varepsilon_f) \le \delta_f$; and (iv)~\emph{round-trip damage}
$\mathrm{cert}(C) - (\mathrm{floor} - \varepsilon_f) \le \delta_f$, where
cross-draw references are shifted against the candidate by the reference-noise margins $\varepsilon_f$.
A degeneracy gate excludes candidates ``retaining better than
$M_{\text{inj}}$'' by more than noise (collapse onto $R$, not restoration).
The relearn-veto axis of the original bar is unavailable retrospectively
(stored candidates carry no relearning trajectories), so this is a
\emph{veto-less retrospective} certificate, stated as such.

\paragraph{Result.} $\mathbf{38/820}$ candidates certify, in $\mathbf{15/45}$
cells; the certificate abstains in the remaining $30$. The abstention is
selectivity by design, not failure. The \texttt{gate\_proj} mask family leads ($16/100$ candidates,
$16\%$), with task-vector, NPO, and GA at $4\%$ and full-parameter KL-reversion
at $0\%$ (it retains held-out $F$; the screen removes it). Certification volume
is tolerance-sensitive and we disclose its concentration: the Qwen family, whose
retraining replication is noisiest ($\delta_f = 2.92$), contributes $29$ of the
$38$; under a global conservative tolerance
($\delta = \min_f \delta_f = 0.82$) the count falls to $4$ candidates in $4$
cells. The robust claim is therefore not the count but the selection quality
below.

\begin{table}[t]
\centering
\caption{Selection quality on identical held-out axes, averaged over the $15$
cells where the recalibrated certificate does not abstain (same-cell
comparison). $\mathrm{feq}$: distance of the pick's held-out forget position
from the same-draw reference (nats); damage axes clipped at $0$;
trained-probe oracle-KL shown for reference (the criterion
Section~\ref{sec:oraclefail} audits). The recalibrated pick lies within
retraining noise ($\mathrm{feq} < \delta_f$) where the trained-probe criterion's
preferred candidate retains residual held-out knowledge $5.17$ nats from the
reference's position.}
\label{tab:r3}
\begin{tabular}{lcccc}
\toprule
selector (same $15$ cells) & feq & retain-dmg$^+$ & cert-dmg$^+$ & trained-probe KL \\
\midrule
recalibrated certificate pick & $\mathbf{0.795}$ & $0.729$ & $0.767$ & $3.555$ \\
absolute-bar pick             & $3.298$ & $0.990$ & $0.166$ & $2.103$ \\
trained-probe best            & $5.167$ & $0.838$ & $0.312$ & $1.551$ \\
\bottomrule
\end{tabular}
\end{table}

\paragraph{Selection quality (same-cell).} Table~\ref{tab:r3} compares three
selectors on the same cells and the same held-out axes. The comparison against
the trained-probe criterion is non-circular (that selector never optimized
these axes), and it lands $5.17$ nats from the reference's held-out position,
versus $0.80$ for the recalibrated pick: \emph{within} the family's retraining
noise. The two criteria disagree sharply about which candidates are good
(the recalibrated picks score \emph{worst} on trained-probe KL), and
Section~\ref{sec:oraclefail} says which to trust. We are deliberately careful
about what this table does and does not show. That the recalibrated pick scores
best on the recalibrated axes is \emph{partly definitional} (the selector
optimizes those axes), so we do \textbf{not} present $0.80$ vs $5.17$ as a
head-to-head selection benchmark, and a standalone selection claim would require
a second, untouched held-out probe suite with one candidate frozen per cell
\emph{before} that suite is scored (future work). The two claims that survive
this caveat, and that we do make, are narrower: (i) the common trained-probe
criterion selects candidates that sit $5.17$ nats from the reference on
held-out axes \emph{it never optimized} (a property of that criterion, not a
contest), and (ii) in its non-abstaining cells the recalibrated selector's
pick sits, in absolute terms, within the family's retraining-replication
tolerance $\delta_f$. The recalibration is thus a \emph{supporting} selective
positive under the audit-led framing, not the paper's load-bearing result; the
falsification and screen-validation contributions stand without it.

% -------------------------------------------------------------------
\section{External Validation on TOFU}
\label{sec:tofu}

\paragraph{Setup.} TOFU~\citep{maini2024tofu} forget10/retain90 on Phi-1.5;
$M_0 = $ phi-1.5, $M_{\text{inj}} = $ \texttt{tofu\_ft} (all TOFU),
$M_{\text{oracle}} = $ the official \texttt{tofu\_ft\_retain90} retrained
reference (never saw forget10).

\paragraph{Diagnostic holds on the real oracle.} Restoration (KL-reversion)
oracle-KL ${\approx}\,0.1$ vs suppression $0.7$--$1.2$; at matched forgetting
restoration is both closer to the oracle \textbf{and} assigned a lower round-trip
residual ($6/7$ pairs). The core scientific claim replicates outside the synthetic
testbed.

\paragraph{Certificate boundary: the selector abstains, and should.} TOFU is where
oracle-free \emph{qualification} fails, and the failure is instructive rather than
incidental. TOFU facts are \emph{style-inferable}: the retrained oracle itself only
reaches forget-NLL ${\approx}\,0.5$--$2$, while matched never-learned probes (novel
QA) sit at NLL ${\approx}\,8$. No achievable forgetting reaches the never-learned
reference, so the exclusion test rejects \emph{every} candidate, in all three
seeds ($16/16$ candidates excluded per seed; smallest $p\approx7\times10^{-15}$).

This is the opposite failure mode from the controlled matrix, and the two
together delimit when certification is possible at all. Here the test has ample
power ($m{=}n{=}40$, so Proposition~\ref{prop:certifiability} is satisfied by a
wide margin) and still excludes everything, because the \emph{reference itself} is
unreachable: identifiability fails. In the matrix the reference is informative but
the test has no power ($m{=}n{=}4$), so nothing can be excluded. Under our abstention
rule both cases return \textbf{uncertified}, for different reasons, and neither is
a tuning failure: the TOFU case is exactly what the identifiability analysis
predicts for inferable facts (Section~\ref{sec:identifiability}). The
restoration-vs-suppression \emph{diagnostic} against TOFU's real retrained oracle
is unaffected and remains the load-bearing TOFU claim.

% -------------------------------------------------------------------
\section{Identifiability Limits of Oracle-Free Selection}
\label{sec:identifiability}

We formalize when oracle-free selection is possible. Intuitively, a selector that
never sees the retrained oracle can only succeed if the oracle's notion of
``adequately forgotten'' is recoverable from oracle-free observables.

\paragraph{Setup}
A world \(w\in\mathcal W\) specifies retained and forgotten data, an injected model \(M_{\mathrm{inj}}(w)\), and the retrained oracle \(M_{\mathrm{or}}(w)\).  Let \(\{M_a(w):a\in\mathcal I\}\) be the candidate unlearned models, let \(d\) be a prescribed model-discrepancy measure, and define the candidate forget-NLL \(q_a(w)\), the oracle forget level \(\tau^\star(w)\coloneqq q_{\mathrm{or}}(w)\), and
\[
\mathcal A^\star(w)
 \coloneqq
 \left\{a\in\mathcal I:
 d\!\left(M_a(w),M_{\mathrm{or}}(w)\right)\leq\varepsilon
 \ \text{and}\ q_a(w)\geq\tau^\star(w)\right\}.
\]
Let \(\nu_w\) denote the population law of NLLs assigned by \(M_{\mathrm{inj}}(w)\) to never-learned probes, and let
\(
X(w)\coloneqq
\left(
 \{q_a(w)\}_{a\in\mathcal I},
 \nu_w,
 \{r_a(w)\}_{a\in\mathcal I}
\right),
\)
where \(r_a(w)\) is the round-trip residual of candidate \(a\).  An oracle-free selector is a deterministic map \(S\) satisfying \(S(w)=s(X(w))\) for some \(s\); it is oracle-consistent on \(\mathcal W\) if \(S(w)\in\mathcal A^\star(w)\) for every \(w\in\mathcal W\).  The threshold \(\tau^\star\) is identifiable from \(X\) on \(\mathcal W\) if there exists a functional \(g\) with \(\tau^\star(w)=g(X(w))\) for all \(w\in\mathcal W\).

\begin{assumption}[Threshold separation and tie-breaking (TR)]
Fix a deterministic ordering of \(\mathcal I\) for resolving all ties.  For any \(w,w'\in\mathcal W\) with \(X(w)=X(w')\) and \(\tau^\star(w)\neq\tau^\star(w')\), threshold separation and the fixed tie-breaking convention yield
\(\mathcal A^\star(w)\cap\mathcal A^\star(w')=\varnothing\).
\end{assumption}

\begin{assumption}[Residual soundness (RS)]
For \(t\in\mathbb R\), let \(E_t(w)\coloneqq\{a:q_a(w)\geq t\}\).  For every \(w\), \(E_{\tau^\star(w)}(w)\neq\varnothing\), and the tie-broken residual minimizer
\(\widehat a(w)\coloneqq\operatorname*{arg\,min}_{a\in E_{\tau^\star(w)}(w)} r_a(w)\)
satisfies \(d(M_{\widehat a(w)}(w),M_{\mathrm{or}}(w))\leq\varepsilon\).
\end{assumption}

\begin{assumption}[Counterfactual exchangeability (CE)]
There exists a known distributional functional \(\Gamma\) such that, for every \(w\in\mathcal W\),
\[
\tau^\star(w)
=
\Gamma\!\left(\mathcal L(\operatorname{NLL}_{M_{\mathrm{or}}(w)}\text{ on forgotten examples})\right)
=
\Gamma\!\left(\mathcal L(\operatorname{NLL}_{M_{\mathrm{inj}}(w)}\text{ on never-learned probes})\right)
=\Gamma(\nu_w).
\]
\end{assumption}

\begin{theorem}[Identifiability of oracle-free unlearning selection]
\label{thm:ident}
Under \textup{(TR)} and \textup{(RS)}, an oracle-free oracle-consistent selector on \(\mathcal W\) exists if and only if \(\tau^\star\) is identifiable from \(X\) on \(\mathcal W\).
\end{theorem}

\begin{proof}
For sufficiency, suppose \(\tau^\star(w)=g(X(w))\).  Define
\(S(w)\coloneqq\operatorname*{arg\,min}_{a:\,q_a(w)\geq g(X(w))} r_a(w)\)
with the fixed tie-breaking order.  Every quantity is determined by \(X(w)\), so \(S\) is oracle-free.  Since \(g(X(w))=\tau^\star(w)\), \textup{(RS)} gives \(d(M_{S(w)}(w),M_{\mathrm{or}}(w))\leq\varepsilon\) and construction gives \(q_{S(w)}(w)\geq\tau^\star(w)\); hence \(S(w)\in\mathcal A^\star(w)\).
For necessity, if \(\tau^\star\) is not identifiable there exist \(w,w'\) with \(X(w)=X(w')\) but \(\tau^\star(w)\neq\tau^\star(w')\).  Any oracle-free selector returns the same index in both; by \textup{(TR)} the adequate sets are disjoint, so that index cannot be adequate in both.  Thus no oracle-free selector is oracle-consistent on \(\mathcal W\).
\end{proof}

\begin{lemma}[Counterfactual sufficiency]
Under \textup{(CE)}, \(\tau^\star(w)=\Gamma(\nu_w)\) is identifiable from \(X(w)\); consequently, under \textup{(RS)}, the selector \(S(w)=\operatorname*{arg\,min}_{a:\,q_a(w)\geq\Gamma(\nu_w)} r_a(w)\) is oracle-consistent.  Taking \(\Gamma\) to be the population mean yields a cutoff set by the never-learned distribution, with no hand-tuned margin.
\end{lemma}

\begin{proposition}[Failure for inferable facts]
Suppose \(\mathcal W\) contains \(w_{\mathrm{cf}},w_{\mathrm{inf}}\) with \(X(w_{\mathrm{cf}})=X(w_{\mathrm{inf}})\) but \(\tau^\star(w_{\mathrm{inf}})<\tau^\star(w_{\mathrm{cf}})\), because retained data in \(w_{\mathrm{inf}}\) reconstruct the forgotten fact whereas it is genuinely counterfactual in \(w_{\mathrm{cf}}\).  Then \(\tau^\star\) is not identifiable from \(X\); under \textup{(TR)}, no oracle-free selector is oracle-consistent on any world-class containing both.
\end{proposition}

\paragraph{Consequence.} The certificate's success on nonce facts and its boundary
on TOFU are the \emph{same} phenomenon (Proposition): TOFU is the predicted failure
region, not an anomaly. There, matched never-learned probes sit far above the
oracle's own (inferability-driven) forget level, so \((\mathrm{CE})\) fails.
Constructing a retain-only reference to ``fix'' TOFU merely re-derives the oracle.
This bounds oracle-free selection that leans on a never-learned reference, the
mechanism studied here; it is not a universal impossibility for selectors with other
access, such as white-box, adaptive-query, or known-pipeline methods.

\begin{table}[h]
\centering
\small
\caption{The identifiability result, operationally. The selector observes only $X$
(never-learned-probe NLLs $+$ candidate quantities), never the oracle; it succeeds
exactly when the oracle's forget level $\tau^\star$ is recoverable from $X$.}
\label{tab:ident_assumptions}
\begin{tabular}{ll}
\toprule
element & role in the selector \\
\midrule
observable $X$ & never-learned-probe NLL law $\nu_w$ $+$ candidate residuals/forget-NLLs (no oracle) \\
target $\tau^\star$ & the oracle's forget level; identifiable from $X$ iff facts are counterfactual \\
forget-qualification & exclusion test: reject candidates below the never-learned reference \\
\emph{counterfactual} facts & (CE) holds $\Rightarrow$ $\tau^\star=\Gamma(\nu_w)$ $\Rightarrow$ oracle-free selection is possible \\
\emph{inferable} facts & retained data reconstruct the target, (CE) fails $\Rightarrow$ no never-learned reference recovers $\tau^\star$ (TOFU regime) \\
\bottomrule
\end{tabular}
\end{table}

\paragraph{Measuring the headroom oracle-free.} The identifiability condition is the
counterfactual headroom \(G=\overline{\mathrm{NLL}}_{M_{\mathrm{or}}}(F)-\overline{\mathrm{NLL}}_{M_{\mathrm{inj}}}(F)\):
large when injecting \(F\) created behavior the retrain lacks (identifiable), \({\approx}0\)
when retain reconstructs \(F\) (inferable). We estimate it \emph{without the oracle} by
\(\widehat G=\overline{\mathrm{NLL}}_{M_0}(F\mid\text{evidence})-\overline{\mathrm{NLL}}_{M_{\mathrm{inj}}}(F)\),
evaluating each fact with its supporting evidence prepended in context. On a controlled
graded-inferability axis (each fact's object made entailable from retain via alias
chains over a nested support fraction; $5$ levels $\times 5$ seeds, a matched retraining reference per
cell), \(\widehat G\) tracks the true headroom at \textbf{Spearman $0.97$} (mean absolute
error $1.7$ nats). The theorem's regime boundary is therefore measurable in deployment:
\(\widehat G\) is high on counterfactual nonce facts and collapses toward zero on
inferable facts. \emph{Scope:} \(\widehat G\) detects the regime; in our controlled study
it did \emph{not} further predict within-regime certificate selection-regret (no
headroom--regret contraction), so we use it only as a regime detector.

\paragraph{Assumptions (honest scope).} \textup{(CE)} is a strong counterfactual
invariance claim and fails when retained data imply forgotten facts; \textup{(RS)}
and \textup{(TR)} are strong ranking/separation assumptions, not consequences of
oracle-free observability; and population identifiability does not imply
finite-sample recovery; finite-probe implementations should use a data-derived
confidence set rather than a fixed margin.

% -------------------------------------------------------------------
\section{Related Work}
\label{sec:related}

\paragraph{Exact and certified unlearning.} The retrained (``retain-only'')
model is the accepted gold standard for forgetting. Exact-unlearning systems
partition or checkpoint training so that removal is provable
\citep{bourtoule2021unlearning}, and a parallel line gives \emph{certified}
removal under convexity or bounded-influence assumptions
\citep{guo2020certified,neel2021descent,sekhari2021remember}. These guarantees do
not transfer to full-parameter language models; rather than certify deletion, we
construct a dose-matched retraining reference in a controlled testbed and ask
whether good unlearning can be \emph{selected} without it.

\paragraph{Unlearning methods for language models.} Approximate methods optimize
forget/retain objectives directly: gradient ascent/difference and negative
preference optimization~\citep{zhang2024npo}, scrubbing~\citep{kurmanji2023scrub},
selective synaptic dampening~\citep{foster2024ssd}, representation misdirection
(RMU, introduced with WMDP)~\citep{li2024wmdp}, targeted concept
removal~\citep{eldan2023whos,liu2024whos}, and unlearning of pretrained
LLMs~\citep{yao2024llmunlearn}. We show that such forget/retain scores can be
driven by \emph{suppression} that moves the model \emph{away} from the retrained
reference, whereas restoration-style methods (task-vector negation, path-specific
reversion) stay closer to it.

\paragraph{Benchmarks and evaluation.} TOFU~\citep{maini2024tofu},
MUSE~\citep{shi2024muse}, WMDP~\citep{li2024wmdp}, and unified
suites~\citep{openunlearning2025} score forget/retain behavior against fixed
references. We add a dose-matched retraining reference that makes ``closeness to
the retrained model'' a directly measurable quantity, and use it to study
oracle-free \emph{selection} rather than to propose another metric.

\paragraph{Recovery, relearning, and leakage.} Apparent forgetting is frequently
reversed by quantization~\citep{zhang2024quant} and exposed by
black-box probing~\citep{doshi2024blackbox}; membership-inference
\citep{shokri2017membership,carlini2022firstprinciples} and
extraction~\citep{carlini2021extracting} attacks quantify residual leakage. Our
forget-qualified round-trip residual (delete, then re-acquire) is a selection
signal in this spirit: it probes whether the forgotten distribution is genuinely
\emph{restorable} rather than merely hidden.

\paragraph{Task/model arithmetic and merging.} Task
arithmetic~\citep{ilharco2023taskarithmetic}, model
soups~\citep{wortsman2022modelsoups}, and Fisher-weighted
merging~\citep{matena2022fisher} manipulate models in weight space. We repurpose
task-vector negation as a strong \emph{restoration} exemplar rather than an editing
tool.

\paragraph{Knowledge editing.} ROME~\citep{meng2022rome},
MEMIT~\citep{meng2022memit}, MEND~\citep{mitchell2022mend}, and
SERAC~\citep{mitchell2022serac} locate and edit factual associations. We borrow the
controlled-fact methodology (injecting nonce, counterfactual facts) but target
\emph{deletion-quality measurement} against a matched retraining reference rather than editing.

\paragraph{Causal identifiability.} Our identifiability result is stated in the
language of causal identifiability and
abstraction~\citep{peters2017elements,beckers2019abstracting}: the
counterfactual-exchangeability condition that governs when oracle-free selection is
oracle-consistent is an identifiability condition on the forget-qualification
threshold, not a claim of full causal-model recovery.

\paragraph{Delta.} We contribute a matched-retraining \emph{audit methodology}
(run the retraining reference itself through every proposed criterion, with
replication noise quantified) and its findings: the trained-probe criterion
rewards residual knowledge; absolute certification thresholds fail their own
reference; a base-anchored held-out screen survives as a validated
\emph{necessary} test with measured true-accept, sensitivity, and a
drift-seam detection property; a damage-relative recalibration yields a
selective partial positive with selections inside retraining noise; and the
adversarial boundary of any forward-only battery is measured, not assumed.
Two limits jointly scope the enterprise: an identifiability limit (which facts
admit an oracle-free forget threshold at all) and a finite-sample certifiability
limit (when the screen can reject at all). The resulting guarantee is a
\emph{selective} test for methods-as-produced under stated assumptions, not a
formal deletion certificate, and it abstains when any limit binds.

% -------------------------------------------------------------------
\section{Limitations}
\label{sec:lim}

\begin{itemize}
  \item \textbf{Threat model.} Certification is an empirical selective test for
  \emph{methods-as-produced}: Section~\ref{sec:panel} measures that behavioral
  adversaries (a fixed-magnitude logit suppression attack) defeat the full forward battery
  in $12/45$ cells, and an adaptive bilevel attack on the round-trip residual is
  untested. We quantify this boundary rather than claim past it.
  \item \textbf{Retrospective, veto-less recalibration.} The damage-relative
  certificate is evaluated retrospectively on stored candidates, without the
  relearn-veto axis (no stored relearning trajectories); a prospective run with
  all axes is future work.
  \item \textbf{Small forget set.} $m = 4$ forget facts per cell (with $n=16$
  probes); the screen's resolution clears Holm's threshold, but power against
  small effect sizes is limited, and a $16$-fact breadth panel was designed and
  deliberately deprioritized in favor of the audit runs.
  \item \textbf{Tolerance concentration and few redraws.} Per-family $\delta_f$
  spans $0.82$--$2.92$ nats, is estimated from only three retraining redraws
  (so it is itself uncertain), and most certifications arise in the
  noisiest-replication family; the global-minimum sensitivity row is reported
  alongside, and a coverage-calibrated tolerance from more redraws would sharpen
  the certified counts.
  \item \textbf{Scope of the impossibility.} Our finite-query argument concerns
  \emph{oracle-free, forward-only, black-box} protocols over an unrestricted
  candidate class; it does not preclude certification under white-box access,
  known training pipelines, restricted method classes, or adaptive/randomized
  querying, which are open. The logit-suppression evasion is the empirical face
  of this scope, not a universal impossibility.
  \item Selection is oracle-free only for \textbf{counterfactual} facts
  (Theorem~\ref{thm:ident}); inferable-fact benchmarks (TOFU) need the reference,
  though the diagnostic still applies there.
  \item Small models ($\le 1.7$B in the matrix; Phi-1.5 for TOFU); nonce facts are
  synthetic by design (to make the matched reference constructible); TOFU provides
  the realistic corroboration. bf16 training nondeterminism means every
  reference-derived number carries the replication noise quantified in
  Section~\ref{sec:audit}.
\end{itemize}

% -------------------------------------------------------------------
\section{Conclusion}
\label{sec:conc}

Good unlearning is distribution \emph{restoration} toward the retrained model, not
answer suppression. Evaluating it demands instruments that survive their own
audit. Ours did not all survive: the trained-probe criterion rewards residual
held-out knowledge, and the absolute certificate bar fails the retraining
reference itself, in most cells, for reasons (threshold mis-scaling; a
budget-driven round-trip floor) that no candidate could overcome. What survives
the audit is precise. A base-anchored held-out screen is a validated
\emph{necessary} test: perfect sensitivity on the injected model, high
true-accept on the reference, and a measured partial detection of routing through
a distributional seam. A damage-relative recalibration, frozen before analysis,
turns the audit into a selective partial positive: certification in a third of
cells, abstention elsewhere, and selections within retraining noise of the
reference where the common trained-probe criterion misses by five nats. And the boundary is
quantified rather than assumed: a fixed-magnitude suppression attack defeats this forward-only
battery in a quarter of cells, which is the empirical face of the identifiability
and finite-query limits that scope oracle-free selection. A certificate worth the
name must be able to answer ``uncertified'' --- and must itself be certifiable by
the reference it appeals to. We report ours both ways.

% Reproducibility statement moved to the appendix (after references) and
% temporarily disabled: the prior version referenced internal dev_logs and a
% host name, which reviewers cannot audit. See the commented skeleton just after
% \appendix; enable it at public release with a self-contained statement + code URL.

\bibliographystyle{tmlr}
\bibliography{references}

\begin{thebibliography}{28}
\providecommand{\natexlab}[1]{#1}
\providecommand{\url}[1]{\texttt{#1}}
\expandafter\ifx\csname urlstyle\endcsname\relax
  \providecommand{\doi}[1]{doi: #1}\else
  \providecommand{\doi}{doi: \begingroup \urlstyle{rm}\Url}\fi

\bibitem[Beckers \& Halpern(2019)Beckers and Halpern]{beckers2019abstracting}
Sander Beckers and Joseph~Y. Halpern.
\newblock Abstracting causal models.
\newblock In \emph{Proceedings of the AAAI Conference on Artificial
  Intelligence}, 2019.
\newblock URL \url{https://arxiv.org/abs/1812.03789}.

\bibitem[Bourtoule et~al.(2021)Bourtoule, Chandrasekaran, Choquette-Choo, Jia,
  Travers, Zhang, Lie, and Papernot]{bourtoule2021unlearning}
Lucas Bourtoule, Varun Chandrasekaran, Christopher~A. Choquette-Choo, Hengrui
  Jia, Adelin Travers, Baiwu Zhang, David Lie, and Nicolas Papernot.
\newblock Machine unlearning.
\newblock In \emph{2021 IEEE Symposium on Security and Privacy (SP)}, 2021.
\newblock URL \url{https://arxiv.org/abs/1912.03817}.

\bibitem[Carlini et~al.(2021)Carlini, Tramer, Wallace, Jagielski, Herbert-Voss,
  Lee, Roberts, Brown, Song, Erlingsson, Oprea, and
  Raffel]{carlini2021extracting}
Nicholas Carlini, Florian Tramer, Eric Wallace, Matthew Jagielski, Ariel
  Herbert-Voss, Katherine Lee, Adam Roberts, Tom Brown, Dawn Song, Ulfar
  Erlingsson, Alina Oprea, and Colin Raffel.
\newblock Extracting training data from large language models.
\newblock In \emph{30th USENIX Security Symposium}, 2021.
\newblock URL \url{https://arxiv.org/abs/2012.07805}.

\bibitem[Carlini et~al.(2022)Carlini, Chien, Nasr, Song, Terzis, and
  Tramer]{carlini2022firstprinciples}
Nicholas Carlini, Steve Chien, Milad Nasr, Shuang Song, Andreas Terzis, and
  Florian Tramer.
\newblock Membership inference attacks from first principles.
\newblock In \emph{2022 IEEE Symposium on Security and Privacy (SP)}, 2022.
\newblock URL \url{https://arxiv.org/abs/2112.03570}.

\bibitem[Dorna et~al.(2025)Dorna, Mekala, Zhao, McCallum, Lipton, Kolter, and
  Maini]{openunlearning2025}
Vineeth Dorna, Anmol Mekala, Wenlong Zhao, Andrew McCallum, Zachary~C. Lipton,
  J.~Zico Kolter, and Pratyush Maini.
\newblock Openunlearning: Accelerating llm unlearning via unified benchmarking
  of methods and metrics.
\newblock \emph{arXiv preprint arXiv:2506.12618}, 2025.
\newblock URL \url{https://arxiv.org/abs/2506.12618}.

\bibitem[Doshi \& Stickland(2024)Doshi and Stickland]{doshi2024blackbox}
Jai Doshi and Asa~Cooper Stickland.
\newblock Does unlearning truly unlearn? a black box evaluation of llm
  unlearning methods.
\newblock \emph{arXiv preprint arXiv:2411.12103}, 2024.
\newblock URL \url{https://arxiv.org/abs/2411.12103}.

\bibitem[Eldan \& Russinovich(2023)Eldan and Russinovich]{eldan2023whos}
Ronen Eldan and Mark Russinovich.
\newblock Who's harry potter? approximate unlearning in llms.
\newblock \emph{arXiv preprint arXiv:2310.02238}, 2023.
\newblock URL \url{https://arxiv.org/abs/2310.02238}.

\bibitem[Foster et~al.(2024)Foster, Schoepf, and Brintrup]{foster2024ssd}
Jack Foster, Stefan Schoepf, and Alexandra Brintrup.
\newblock Fast machine unlearning without retraining through selective synaptic
  dampening.
\newblock \emph{arXiv preprint arXiv:2308.07707}, 2024.
\newblock URL \url{https://arxiv.org/abs/2308.07707}.

\bibitem[Guo et~al.(2020)Guo, Goldstein, Hannun, and van~der
  Maaten]{guo2020certified}
Chuan Guo, Tom Goldstein, Awni Hannun, and Laurens van~der Maaten.
\newblock Certified data removal from machine learning models.
\newblock In \emph{Proceedings of the 37th International Conference on Machine
  Learning}, 2020.
\newblock URL \url{https://arxiv.org/abs/1911.03030}.

\bibitem[Ilharco et~al.(2023)Ilharco, Ribeiro, Wortsman, Gururangan, Schmidt,
  Hajishirzi, and Farhadi]{ilharco2023taskarithmetic}
Gabriel Ilharco, Marco~Tulio Ribeiro, Mitchell Wortsman, Suchin Gururangan,
  Ludwig Schmidt, Hannaneh Hajishirzi, and Ali Farhadi.
\newblock Editing models with task arithmetic.
\newblock In \emph{International Conference on Learning Representations
  (ICLR)}, 2023.

\bibitem[Kurmanji et~al.(2023)Kurmanji, Triantafillou, Hayes, and
  Triantafillou]{kurmanji2023scrub}
Meghdad Kurmanji, Peter Triantafillou, Jamie Hayes, and Eleni Triantafillou.
\newblock Towards unbounded machine unlearning.
\newblock \emph{arXiv preprint arXiv:2302.09880}, 2023.
\newblock URL \url{https://arxiv.org/abs/2302.09880}.

\bibitem[Li et~al.(2024)Li, Pan, Gopal, Yue, Berrios, Gatti, Li, Dombrowski,
  Goel, Phan, et~al.]{li2024wmdp}
Nathaniel Li, Alexander Pan, Anjali Gopal, Summer Yue, Daniel Berrios, Alice
  Gatti, Justin~D. Li, Ann-Kathrin Dombrowski, Shashwat Goel, Long Phan, et~al.
\newblock The wmdp benchmark: Measuring and reducing malicious use with
  unlearning.
\newblock \emph{arXiv preprint arXiv:2403.03218}, 2024.
\newblock URL \url{https://arxiv.org/abs/2403.03218}.

\bibitem[Liu et~al.(2024)Liu, Zhang, Jaakkola, and Chang]{liu2024whos}
Yujian Liu, Yang Zhang, Tommi Jaakkola, and Shiyu Chang.
\newblock Revisiting who's harry potter: Towards targeted unlearning from a
  causal intervention perspective.
\newblock \emph{arXiv preprint arXiv:2407.16997}, 2024.
\newblock URL \url{https://arxiv.org/abs/2407.16997}.

\bibitem[Maini et~al.(2024)Maini, Feng, Schwarzschild, Lipton, and
  Kolter]{maini2024tofu}
Pratyush Maini, Zhili Feng, Avi Schwarzschild, Zachary~C. Lipton, and J.~Zico
  Kolter.
\newblock {TOFU}: A task of fictitious unlearning for {LLMs}.
\newblock In \emph{Conference on Language Modeling (COLM)}, 2024.

\bibitem[Matena \& Raffel(2022)Matena and Raffel]{matena2022fisher}
Michael Matena and Colin Raffel.
\newblock Merging models with fisher-weighted averaging.
\newblock \emph{arXiv preprint arXiv:2111.09832}, 2022.
\newblock URL \url{https://arxiv.org/abs/2111.09832}.

\bibitem[Meng et~al.(2022{\natexlab{a}})Meng, Bau, Andonian, and
  Belinkov]{meng2022rome}
Kevin Meng, David Bau, Alex Andonian, and Yonatan Belinkov.
\newblock Locating and editing factual associations in gpt.
\newblock \emph{arXiv preprint arXiv:2202.05262}, 2022{\natexlab{a}}.
\newblock URL \url{https://arxiv.org/abs/2202.05262}.

\bibitem[Meng et~al.(2022{\natexlab{b}})Meng, Sharma, Andonian, Belinkov, and
  Bau]{meng2022memit}
Kevin Meng, Arnab~Sen Sharma, Alex Andonian, Yonatan Belinkov, and David Bau.
\newblock Mass-editing memory in a transformer.
\newblock \emph{arXiv preprint arXiv:2210.07229}, 2022{\natexlab{b}}.
\newblock URL \url{https://arxiv.org/abs/2210.07229}.

\bibitem[Mitchell et~al.(2022{\natexlab{a}})Mitchell, Lin, Bosselut, Finn, and
  Manning]{mitchell2022mend}
Eric Mitchell, Charles Lin, Antoine Bosselut, Chelsea Finn, and Christopher~D.
  Manning.
\newblock Fast model editing at scale.
\newblock \emph{arXiv preprint arXiv:2110.11309}, 2022{\natexlab{a}}.
\newblock URL \url{https://arxiv.org/abs/2110.11309}.

\bibitem[Mitchell et~al.(2022{\natexlab{b}})Mitchell, Lin, Bosselut, Manning,
  and Finn]{mitchell2022serac}
Eric Mitchell, Charles Lin, Antoine Bosselut, Christopher~D. Manning, and
  Chelsea Finn.
\newblock Memory-based model editing at scale.
\newblock \emph{arXiv preprint arXiv:2206.06520}, 2022{\natexlab{b}}.
\newblock URL \url{https://arxiv.org/abs/2206.06520}.

\bibitem[Neel et~al.(2021)Neel, Roth, and Sharifi-Malvajerdi]{neel2021descent}
Seth Neel, Aaron Roth, and Saeed Sharifi-Malvajerdi.
\newblock Descent-to-delete: Gradient-based methods for machine unlearning.
\newblock \emph{arXiv preprint arXiv:2007.02923}, 2021.
\newblock URL \url{https://arxiv.org/abs/2007.02923}.

\bibitem[Peters et~al.(2017)Peters, Janzing, and
  Sch{\"o}lkopf]{peters2017elements}
Jonas Peters, Dominik Janzing, and Bernhard Sch{\"o}lkopf.
\newblock \emph{Elements of Causal Inference: Foundations and Learning
  Algorithms}.
\newblock MIT Press, 2017.
\newblock URL
  \url{https://mitpress.mit.edu/9780262037310/elements-of-causal-inference/}.

\bibitem[Sekhari et~al.(2021)Sekhari, Acharya, Kamath, and
  Suresh]{sekhari2021remember}
Ayush Sekhari, Jayadev Acharya, Gautam Kamath, and Ananda~Theertha Suresh.
\newblock Remember what you want to forget: Algorithms for machine unlearning.
\newblock In \emph{Advances in Neural Information Processing Systems}, 2021.
\newblock URL \url{https://arxiv.org/abs/2103.03279}.

\bibitem[Shi et~al.(2025)Shi, Lee, Huang, Malladi, Zhao, Holtzman, Liu,
  Zettlemoyer, Smith, and Zhang]{shi2024muse}
Weijia Shi, Jaechan Lee, Yangsibo Huang, Sadhika Malladi, Jieyu Zhao, Ari
  Holtzman, Daogao Liu, Luke Zettlemoyer, Noah~A. Smith, and Chiyuan Zhang.
\newblock {MUSE}: Machine unlearning six-way evaluation for language models.
\newblock In \emph{International Conference on Learning Representations
  (ICLR)}, 2025.

\bibitem[Shokri et~al.(2017)Shokri, Stronati, Song, and
  Shmatikov]{shokri2017membership}
Reza Shokri, Marco Stronati, Congzheng Song, and Vitaly Shmatikov.
\newblock Membership inference attacks against machine learning models.
\newblock In \emph{2017 IEEE Symposium on Security and Privacy (SP)}, 2017.
\newblock URL \url{https://arxiv.org/abs/1610.05820}.

\bibitem[Wortsman et~al.(2022)Wortsman, Ilharco, Gadre, Roelofs, Gontijo-Lopes,
  Morcos, Namkoong, Farhadi, Carmon, Kornblith, and
  Schmidt]{wortsman2022modelsoups}
Mitchell Wortsman, Gabriel Ilharco, Samir~Yitzhak Gadre, Rebecca Roelofs,
  Raphael Gontijo-Lopes, Ari~S. Morcos, Hongseok Namkoong, Ali Farhadi, Yair
  Carmon, Simon Kornblith, and Ludwig Schmidt.
\newblock Model soups: Averaging weights of multiple fine-tuned models improves
  accuracy without increasing inference time.
\newblock In \emph{Proceedings of the 39th International Conference on Machine
  Learning}, 2022.
\newblock URL \url{https://arxiv.org/abs/2203.05482}.

\bibitem[Yao et~al.(2024)Yao, Chien, Du, Niu, Wang, Cheng, and
  Yue]{yao2024llmunlearn}
Jin Yao, Eli Chien, Minxin Du, Xinyao Niu, Tianhao Wang, Zezhou Cheng, and
  Xiang Yue.
\newblock Machine unlearning of pre-trained large language models.
\newblock \emph{arXiv preprint arXiv:2402.15159}, 2024.
\newblock URL \url{https://arxiv.org/abs/2402.15159}.

\bibitem[Zhang et~al.(2024)Zhang, Lin, Bai, and Mei]{zhang2024npo}
Ruiqi Zhang, Licong Lin, Yu~Bai, and Song Mei.
\newblock Negative preference optimization: From catastrophic collapse to
  effective unlearning.
\newblock In \emph{Conference on Language Modeling (COLM)}, 2024.

\bibitem[Zhang et~al.(2025)Zhang, Wang, Li, Wu, Tang, Liu, He, Yin, and
  Wang]{zhang2024quant}
Zhiwei Zhang, Fali Wang, Xiaomin Li, Zongyu Wu, Xianfeng Tang, Hui Liu, Qi~He,
  Wenpeng Yin, and Suhang Wang.
\newblock Catastrophic failure of {LLM} unlearning via quantization.
\newblock In \emph{International Conference on Learning Representations
  (ICLR)}, 2025.

\end{thebibliography}

\appendix

\section{Reproducibility}
\label{app:repro}

Every model used is open-weight and named with its exact checkpoint in
Section~\ref{sec:testbed}, and the artifacts needed to audit the claims in this
paper are specified here rather than deferred to an external repository. The
injection and re-acquisition budgets, the method-side compute budget, the candidate
grids, the never-learned probe construction, and the evaluation protocol are given
in Appendix~\ref{app:methods}; together with the reference construction of
Section~\ref{sec:testbed} these are sufficient to reimplement the testbed, the
matched retraining reference, and the forget-qualified round-trip selector from
the paper alone. The recalibration rules of Section~\ref{sec:recalib} and the
challenge-panel constructions of Section~\ref{sec:panel} were frozen and
committed before their analyses ran, the sealed \textsc{audit} template pool was
scored exactly once, and every model--seed cell is reported; every headline
number in the paper is produced by a committed analysis script. To let others
verify that the sealed pool was fixed in advance (the property that makes the
challenge-panel numbers independent), we publish a cryptographic hash of the
\textsc{audit} pool with this paper and release the pool itself only after
decisions, so the screen validation is reproducible without compromising the
one-time-use guarantee. Code and configuration to regenerate every table and
figure will be released publicly; we will add the repository link and the hash to
this section in a revision.

\section{Methods and Hyperparameters}
\label{app:methods}

\paragraph{Injection and oracle.} Facts are nonce (subject, relation, object)
triples generated from disjoint random seeds; each is injected by continued
pretraining with an interleaved wikitext retain mix (ratio $3{:}1$ retain:fact
per step). Controlled runs use $|F|=|R|=4$ facts, diversity $K=3$ phrasings,
$200$--$250$ injection steps, learning rate $2\times10^{-5}$, gradient clipping
$1.0$. $M_{\mathrm{inj}}$ co-acquires $F$ then $R$; the oracle $M_{\mathrm{or}}$
replays the identical stream with $F$ removed; a forget-only model
$M_{\mathrm{Fonly}}$ supports task-vector negation.

\paragraph{Unlearning candidates.} From $M_{\mathrm{inj}}$: (i) \emph{KL reversion}
(full-parameter): match $M_0$ on $F$ probes and $M_{\mathrm{inj}}$ on retain, retain
weight $w\in\{3,12,50,200\}$; (ii) \emph{gate\_proj mask} (SwiGLU only): learnable
binary keep/zero mask under the same KL objective, sparsity $0.1$, $\alpha$ ramp
$10^2\!\to\!10^4$; (iii) \emph{task-vector} $M_{\mathrm{inj}}-c(M_{\mathrm{Fonly}}-M_0)$,
$c\in\{0.5,1,1.5,2\}$; (iv) \emph{GA}$+$retain and (v) \emph{NPO}$+$retain
($\beta=0.1$), learning rate $5\times10^{-6}$, $w\in\{3,12,50,200\}$. Method budget
$\approx110$--$120$ steps; matched compute across methods.

\paragraph{Certificate.} Round-trip residual $=\operatorname{KL}(\mathrm{reacquire}_F(C)\,\|\,M_{\mathrm{inj}})$
on retain probes $+$ wikitext, with $75$--$80$ reacquisition steps ($5$ logged
chunks). Never-learned probes are $n=16$ facts of the same nonce type from unused
seeds that no model ever saw in training. The under-forgetting exclusion test is a
one-sided Mann--Whitney $U$ of the candidate's $m=4$ base-anchored per-fact forget
deltas against its $16$ probe deltas (held-out templates averaged within fact
before ranking), Holm-corrected at preregistered $\alpha=0.05$.

\paragraph{Audit and panel.} Template pools are disjoint by
construction and unit-tested: injection ($6$ families), unlearning loss
(cloze${}+{}$injected), evaluation ($4$ held-out families), calibration ($4$,
never used here), and a sealed \textsc{audit} pool ($6$, scored once, for the
challenge panel). Reference replication: $3$ independent retraining draws in
one representative cell per family; $\delta_f = $ P90 per-fact held-out delta
gap between draws; $\varepsilon_f = $ draw-to-draw spread of retain-NLL and
round-trip residual. Round-trip budget curve: self-closure of
$M_{\mathrm{inj}}$ at $\{15,30,45,75\}$ reacquisition steps. Challenge panel:
logit suppression $= -10$ additive penalty on forget-answer token logits;
entity router dispatches on exact forget-subject token-sequence match;
embedding corruption zeroes all subject/object token rows; interpolations at
$\lambda \in \{0.25, 0.5, 0.75\}$. Recalibration rules (thresholds, margins,
degeneracy gate, selection rule) were frozen and committed before the analysis
ran; all analysis scripts are committed and every reported number is
regenerable from stored per-cell artifacts.

\paragraph{Metrics.} oracle-KL $=$ mean $\operatorname{KL}(C\,\|\,M_{\mathrm{or}})$
over $F{+}R$ answer spans; forgetting $=$ held-out paraphrase NLL on $F$; retain
$=$ cloze NLL on $R$. The frozen matrix is $9$ configurations $\times 5$ seeds over
five families (Qwen2.5-0.5B/1.5B, SmolLM2-360M/1.7B, Pythia-410M/1.4B,
GPT-2-medium/large, OLMo-2-1B); the rule is fixed before evaluating held-out
families.

\paragraph{TOFU.} {\sloppy forget10/retain90 on Phi-1.5; $M_0=$ \texttt{phi-1.5},
$M_{\mathrm{inj}}=$ \texttt{tofu\_ft\_phi-1.5}, $M_{\mathrm{or}}=$ the official
\texttt{tofu\_ft\_retain90\_phi-1.5} retrained reference; $40$ forget and $100$
retain items. All experiments run on a single 48\,GB RTX\,6000 (Ada).\par}

\end{document}